\documentclass[preprint,12pt]{elsarticle}
\usepackage[english]{babel}

\usepackage[letterpaper,top=2cm,bottom=2cm,left=3cm,right=3cm,marginparwidth=1.75cm]{geometry}

\usepackage{amsthm,amsmath,array,amssymb}
\usepackage{bm}
\usepackage{graphicx}
\usepackage[caption=false,font=footnotesize]{subfig}
\usepackage{xspace}
\usepackage{stmaryrd}
\usepackage{xcolor}
\usepackage{float}
\usepackage{textcomp}
\usepackage[pdfencoding=auto]{hyperref}
\usepackage{version}
\excludeversion{hide}
\usepackage[normalem]{ulem}
\useunder{\uline}{\ul}{}
\newcommand\redsout{\bgroup\markoverwith{\textcolor{red}{\rule[0.5ex]{2pt}{0.8pt}}}\ULon}
\usepackage{multirow}

\makeatletter
\newcommand{\AppendixLetterLabel}[1]{%
  \label{#1}
  \begingroup
    \edef\@currentlabel{\Alph{section}}
    \label{#1@letter}
  \endgroup
}
\newcommand{\appLetter}[1]{\ref{#1@letter}}
\makeatother

\newcommand{\bdm}{\begin{displaymath}}
\newcommand{\edm}{\end{displaymath}}

\newcommand{\inner}[2]{\left\langle {#1}, {#2} \right\rangle}

\newcommand{\half}{\frac{1}{2}}

\renewcommand{\(}{\left(}
\renewcommand{\)}{\right)}
\newcommand{\N}{{\mathbb N}}

\newcommand{\Znn}{{\mathbb{Z}_{\geq 0}}}

\newcommand{\Rf}{{\mathbb R}}

\newcommand{\Rnn}{{\mathbb{R}_{\geq 0}}}

\newcommand{\supp}{\operatorname{supp}}

\newcommand{\define}{\, := \,}

\newcommand{\udots}{
  \mathinner {\mkern 1mu\raise 1pt \vbox {\kern 7pt \hbox {.}}\mkern 2mu
  \raise 4pt \hbox {.}\mkern 2mu\raise 7pt \hbox {.}\mkern 1mu}}

\newcommand{\f}{\bm{f}} 
\newcommand{\bg}{\bm{g}}

\newcommand{\bk}{\bm{k}}

\newcommand{\bw}{\bm{w}}

\newcommand{\bxi}{\bm{{\xi}}}

\newcommand{\bphi}{\bm{{\phi}}}

\newcommand{\bpsi}{\bm{{\psi}}}

\newcommand{\bPhi}{\bm{{\Phi}}}

\newcommand{\ndim}{n}

\newcommand{\B}{{\mathcal B}}

\newcommand{\transp}{{\scriptscriptstyle{\mathsf{T}}}}

\newcommand{\bone}{{\boldsymbol 1}}

\DeclareMathOperator{\hl}{hl}

\DeclareMathOperator{\nat}{nat}

\DeclareMathOperator{\diagop}{diag}
\newcommand{\jmax}{{j_\mathrm{max}}}
\newcommand{\pmax}{{p_\mathrm{max}}}
\newcommand{\nmax}{{n_\mathrm{max}}}
\newcommand{\kk}{\kappa}
\newcommand{\Lap}{\mathcal{L}}
\newcommand{\Lrw}{L_\mathrm{rw}}
\newcommand{\Lsym}{L_\mathrm{sym}}
\newcommand{\W}{\mathcal{W}}

\def\ksim{\>\> \text{\stackunder[0pt]{$\sim$}{$\scriptscriptstyle\kk$}} \>\>}
\def\bk{{B_\kk}}
\def\Bk{{B_\kk}}

\usepackage{stackengine}

\newtheorem{theorem}{Theorem}
\newtheorem{remark}{Remark}

\journal{}

\begin{document}
\begin{frontmatter}

\title{Multiscale Hodge Scattering Networks for Data Analysis}
\author{Naoki Saito\corref{cor1}}
\ead{saito@math.ucdavis.edu}

\author{Stefan C. Schonsheck}
\ead{sschonsh@gmail.com}

\author{Eugene Shvarts}
\ead{eugene.shvarts@gmail.com}

\cortext[cor1]{Corresponding author}

\affiliation{organization={Department of Mathematics, University of California, Davis},
            addressline={One Shields Avenue}, 
            city={Davis},
            postcode={95616}, 
            state={CA},
            country={USA}}

\begin{abstract}

We propose new scattering networks for signals measured on simplicial complexes, which we call \emph{Multiscale Hodge Scattering Networks} (MHSNs). 
Our construction builds on multiscale basis dictionaries on simplicial complexes---namely, the $\kappa$-GHWT and $\kappa$-HGLET---which we recently developed for simplices of dimension $\kappa \in \mathbb{N}$ in a given simplicial complex by generalizing the node-based Generalized Haar--Walsh Transform (GHWT) and Hierarchical Graph Laplacian Eigen Transform (HGLET). 
Both the $\kappa$-GHWT and the $\kappa$-HGLET form redundant sets (i.e., dictionaries) of multiscale basis vectors and the corresponding expansion coefficients of a given signal. 
Our MHSNs adopt a layered structure analogous to a convolutional neural network (CNN), cascading the moments of the modulus of the dictionary coefficients.
The resulting features are invariant to reordering of the simplices (i.e., node permutation of the underlying graphs).
Importantly, the use of multiscale basis dictionaries in our MHSNs admits a natural pooling operation---akin to local pooling in CNNs---that can be performed either locally or per scale.
Such pooling operations are more difficult to define in traditional scattering networks based on Morlet wavelets and in geometric scattering networks based on Diffusion Wavelets. 
As a result, our approach extracts a rich set of descriptive yet robust features that can be combined with simple machine learning models (e.g., logistic regression or support vector machines) to achieve high-accuracy classification with far fewer trainable parameters than most modern graph neural networks require.
Finally, we demonstrate the effectiveness of MHSNs on three distinct problem types: signal classification, domain (i.e., graph/simplex) classification, and molecular dynamics prediction.

\end{abstract}


\end{frontmatter}

\section{Introduction}

\emph{Scattering Transforms} were introduced by Mallat in \cite{MALLAT-SCAT} as a method for feature extraction from signals and images. The resulting features are quasi-invariant to translations, stable to certain deformations of the input signals \cite{MALLAT-SCAT, nicola2023stability}, and preserve high-frequency information from the input, making them well-suited for a wide range of data classification tasks, such as texture image classification \cite{BRUNA-MALLAT, CHAK-SAITO-MWSN}. In addition, their computational architecture closely resembles a convolutional neural networks (CNNs), which allows for fast, GPU-friendly computation. In fact, these networks are often thought of as a type of CNN, with predetermined wavelet filter banks as their convolution filters and a pointwise modulus operation as their activation function. A key advantage of these networks over traditional CNNs is that since the filter banks do not need to be learned from input data, they are much less data-hungry. Additionally, they are more interpretable since each channel in the hidden representation is a deterministic cascade of wavelet transform convolutions with nonlinear activation and averaging operators.

More recently, Gao et al.\ introduced an analogous network architecture for node-signals on undirected graphs~\cite{gao2019geometric}, which they termed ``Geometric Scattering (Networks).'' In this context, invariance to node permutation takes the place of translation quasi-\hspace{0pt}invariance in the original scattering transform. This is achieved in a manner similar to PointNet \cite{qi2017pointnet}, by aggregating node features---via either sum pooling or max pooling---into a single value for each channel. The resulting feature extractor is permutation invariant, stable to graph deformations (see also \cite{GAMA-BRUNA-RIBEIRO, zou2020graph}), and can be paired with a simple learning model (typically logistic regression or support vector machine) to achieve near state-of-the-art (SotA) classification results on many datasets, often with far fewer training examples than CNN-based approaches require.  As a result, these networks produce descriptive yet robust features that enable high-accuracy classification with small training sets and a minimal number of training parameters.

In this article, we extend this line of research to signals defined on arbitrarily high-dimensional simplicial structures---edges, triangles, pyramids, and their $\kk$-dimensional analogues. Our methods differ from previous work in two key ways. First, earlier scattering networks have been applied only to point-valued signals, whereas our construction generalizes naturally to higher-dimensional structures. Second, we employ the $\kk$-Hierarchical Graph Laplace Transform ($\kk$-HGLET)~\cite{IRION-SAITO-TSIPN, SAITO-SCHONSHECK-SHVARTS} and the $\kk$-Generalized Haar--Walsh Transform ($\kk$-GHWT)~\cite{IRION-SAITO-GHWT, SAITO-SCHONSHECK-SHVARTS} as the wavelet filter banks in our transforms. By contrast, most prior work has relied on Morlet wavelets for images and Diffusion Wavelets~\cite{COIF-MAGG-DW} for graph-based signals. The bipartition tree induced by the multiscale transforms we proposed in \cite{SAITO-SCHONSHECK-SHVARTS} enables the construction of sparser approximations, which in turn lead to more efficient feature extraction and, consequently, more expressive networks. Moreover, the multiscale structure of these bases allows us to define local pooling operations in a straightforward manner, which can further enhance the performance of scattering networks in many applications. 

\subsection{Comparison with Related Works}
There has been growing interest in studying signals defined on edges, triangles, and higher-dimensional substructures within graph-structured data~\cite{carlsson2009topology, shuman2013emerging, giusti2016two, barbarossa2020topological, chen2021helmholtzian}. Applications in computer vision~\cite{lim2020hodge, roddenberry2022signal}, statistics~\cite{jiang2011statistical}, topological data analysis~\cite{chen2021helmholtzian, schonsheck2023spherical}, and network analysis~\cite{schaub2020random} have benefited from the study of high-dimensional simplicial complexes. Convolution-based simplicial neural networks have shown remarkable results in these domains~\cite{ebli2020simplicial}. We extend this line of research by defining scattering networks on such higher-dimensional domains.

Scattering networks~\cite{BRUNA-MALLAT} were initially introduced as a tool to explain the success of CNNs on many computer vision problems. These networks possess many of the invariance and equivariance properties that make CNNs desirable, but they do not contain any learnable filters; instead, they employ a cascade of wavelet convolutions and contractive nonlinearities. Later, Gao et~al.\ successfully generalized these networks to graphs~\cite{gao2019geometric}. Our work further generalizes these approaches.

The main ways our MHSNs differ from Geometric Scattering Networks (GSNs)~\cite{gao2019geometric} and Deep Haar Scattering Networks (DHSNs)~\cite{cheng2016deep} are: a1) MHSNs accept arbitrary simplicial complexes, whereas GSNs and DHSNs were designed for node signals only; and 2) GSNs and DHSNs are based on the Diffusion Wavelets of Coifman and Maggioni~\cite{COIF-MAGG-DW} and the Haar transform, respectively, and thus are not built on a hierarchical partitioning of the underlying graph. In contrast, MHSNs are constructed over the richer HGLET/GHWT dictionaries, which are more amenable to analysis since they consist of a collection of orthonormal bases (ONBs).

Hodgelets~\cite{roddenberry2022hodgelets} use a kernel defined in the spectral domain---similar to the spectral graph wavelet transform~\cite{HAMMOND-VANDERGHEYNST-GRIBONVAL}---to define another family of wavelet-like frames for signals on simplicial complexes. Topological Slepians~\cite{battiloro2023toposlepians} also form a localized basis dictionary on a given collection of $\kk$-simplices, but their construction is based on maximizing primal-domain concentration of vectors subject to a prescribed dual-domain (frequency-domain) support set. However, both Hodgelets and Topological Slepians are difficult to use for scattering-transform-type representations since they are not hierarchically arranged.

Recently, Chew et~al.\ introduced a method for \emph{windowed} scattering transforms that achieves local-pooling-like operations~\cite{chew2022geometric}. However, because the underlying topology of the graph or complex is non-Euclidean, it may be difficult to consistently define local windows across multiple graphs~\cite{yang2020pfcnn, schonsheck2022parallel}. It may be possible to use the partitioning scheme proposed in~\cite{SAITO-SCHONSHECK-SHVARTS} for these windows, but defining appropriate wavelet families for such a hybrid approach requires further study.

\section{Hodge Laplacians and Multiscale Basis Dictionaries}\label{sec:Dict}

In this section, we review the basic elements of Hodge theory needed to define the Hodge Laplacian on simplicial complexes, and we summarize the construction of multiscale basis functions on these spaces. For a more comprehensive introduction to Hodge theory, see~\cite{carlsson2009topology, giusti2016two, roddenberry2022signal}. For a detailed explanation of multiscale basis dictionaries, see~\cite{IRION-SAITO-TSIPN, SAITO-SCHONSHECK-SHVARTS}.

\subsection{Simplicial Complexes and Boundary Operators}
In this subsection, we review concepts from algebraic topology to formally define simplicial complexes and introduce several notions of adjacency between simplices. Given a vertex (or node) set $V = \{v_1, \ldots, v_n\}$, a \emph{$\kk$-simplex} $\sigma$ is a $(\kk+1)$-subset of $V$. 
A \emph{face} of $\sigma$ is a $\kk$-subset of $\sigma$, so $\sigma$ has exactly $\kk+1$ faces.
A \emph{co-face} of $\sigma$ is a $(\kk+1)$-simplex  of which $\sigma$ is a face. 

A \emph{simplicial complex} $C$ is a collection of simplices closed under taking subsets: if $\sigma \in C$ and $\alpha \subset \sigma$, then $\alpha \in C$.
In particular, if $\sigma \in C$, all of its faces are also in $C$.
Let $\kk_{\mathrm{max}}(C) \define \max\left\{\kk \mid \sigma \in C \ \text{is a $\kk$-simplex}\right\}$. For each $\kk \in \{0, 1,\ldots, \kk_{\mathrm{max}}\}$, let $C_\kk$ denote the set of $\kk$-simplices in $C$, and let $X_\kk$ be the space of real-valued functions on $C_\kk$.
When $\kk > \kk_{\mathrm{max}}$, we set $C_\kk = \emptyset$.
We also refer to $C$ as a \emph{$\kk$-complex} when $\kk_{\mathrm{max}}(C) = \kk$.
A \emph{$\kk$-region} of $C$ is any nonempty subset of $C_\kk$.

Let $C$ be a simplicial complex, and let $\sigma, \tau \in C_\kk$ for some $\kk > 0$.  
- If $\sigma$ and $\tau$ share a face, they are \emph{weakly adjacent}, denoted $\sigma \sim \tau$.  
- If $\sigma \sim \tau$ and they also share a co-face (their \emph{hull}), denoted $\hl(\sigma, \tau)$, and $\hl(\sigma, \tau) \in C$, then $\sigma$ and $\tau$ are \emph{strongly adjacent}, denoted $\sigma \simeq \tau$.  
- If $\sigma \sim \tau$ but $\sigma \not\simeq \tau$ in $C$, then $\sigma$ and $\tau$ are \emph{$\kk$-adjacent}, denoted $\sigma \ksim \tau$.

Figure \ref{fig:twotriangle} illustrates these adjacency relations in a toy $2$-complex.

\begin{figure}[ht] 
    \centering
    \includegraphics[width=.4\textwidth]{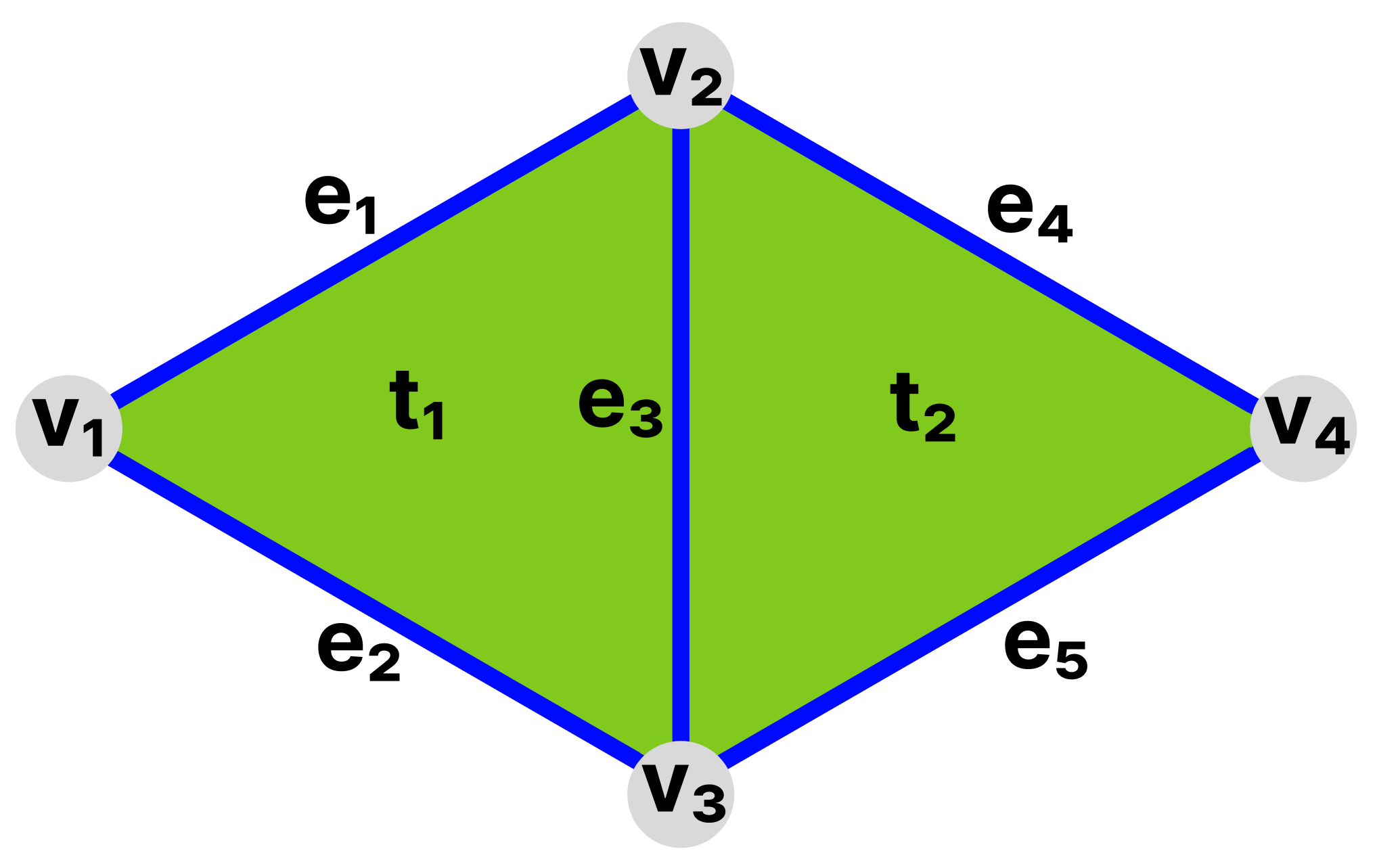}
    \caption{In this small $2$-complex $C$, $e_1 \sim e_4$ because they share the face $v_2$, and $e_1 \sim e_2$ because they share the face $v_1$. Further, $e_1 \simeq e_2$ because their hull $t_1 \in C$, but $e_1 \not\simeq e_4$, so $e_1 \>\> \text{\stackunder[0pt]{$\sim$}{$\scriptscriptstyle 1$}} \>\> e_4$. We have $t_1 \sim t_2$ because they share the face $e_3$, and also $t_1 \>\> \text{\stackunder[0pt]{$\sim$}{$\scriptscriptstyle 2$}}\>\> t_2$.}
    \label{fig:twotriangle}
\end{figure}

Suppose $\sigma = \{v_{i_1}, \ldots, v_{i_{\kk+1}}\}$, $i_1 < \cdots < i_{\kk+1}$, and let $\alpha \subset \sigma$ be a face.
Then $\sigma \setminus \alpha = \{v_{i_{\ell^*}}\}$ for some $\ell^* \in \{1,\ldots,\kk+1\}$.
The \textit{natural parity} of $\sigma$ with respect to $\alpha$ is defined as $\nat(\sigma, \alpha) \define (-1)^{\ell^* + 1}$.
An \emph{oriented} simplex $\sigma$ has an orientation $p_\sigma \in \{\pm 1\}$ indicating whether its parity with its faces agrees with ($+1$) or is opposite to ($-1$) its natural parity.
When $p_\sigma = +1$, we say $\sigma$ is in \emph{natural orientation}.
For example, a directed edge $e = (v_i,v_j)$ with $i<j$ is in natural orientation, while if $i>j$, then $p_e = -1$.
An oriented simplicial complex contains at most one orientation for any given simplex.

Given an oriented simplicial complex $C$, for each $\kk \in \{0, 1, \ldots, \kk_\mathrm{max} \}$,
the \emph{boundary operator} is the linear map $\bk : X_{\kk+1} \to X_\kk$, whose matrix entries for $\sigma \in C_{\kk+1}$ and $\alpha \in C_\kk$ are $[\Bk]_{\alpha\sigma} = p_\sigma \, p_\alpha \, \nat(\sigma, \alpha)$. 
The \emph{coboundary operator} for each $\kk$ is the adjoint $\bk^\transp : X_\kk \to X_{\kk+1}$.
The entries of $\bk$ encode \emph{relative orientation} between a simplex and its faces, and naturally define local signed averaging operators according to adjacency in the simplicial complex.

\subsection{Hodge Laplacian}
\label{sec:hodgelap}
The boundary operators introduced above act as \emph{discrete differential operators} that encode the structure of $\kk$-regions in a simplicial complex, and thus serve as building blocks for the spectral analysis of functions on those regions. 
For analyzing functions on $\kk$-simplices with $\kk>0$, we construct operators based on the \emph{Hodge Laplacian}, or \emph{$\kk$-Laplacian}. 
Following~\cite{lim2020hodge}, the \emph{combinatorial} $\kk$-Laplacian for $\kk$-simplices is defined as
\begin{equation*}
L_\kk \define B_{\kk-1}^\transp B_{\kk-1} + B_\kk B_\kk^\transp\,.
\end{equation*}

Various forms of weighting and normalization are possible, each with its own advantages and challenges, and these choices can also affect how one interprets the Fiedler vector of the resulting Hodge Laplacian; see~\cite[Chap.~4]{SHVARTS-PHD} for a detailed discussion.
In our numerical experiments, we adopt the \textit{symmetrically normalized, weighted} Hodge Laplacian defined in \cite{chen2021helmholtzian} as follows.
For each $\kk \in \{0,1,\ldots,\kk_{\mathrm{max}}\}$, let $D_\kk$ be a diagonal matrix whose entries assign positive weights to the $\kk$-simplices in $C_\kk$.
One choice, used in \cite{chen2021helmholtzian}, is to set each simplex's weight to its degree, defined recursively by $D_{\kk+1} = I$, $\diagop(D_\kk) = |B_\kk| \diagop(D_{\kk+1})$, and $\diagop(D_{\kk-1}) = |B_{\kk-1}| \diagop(D_{\kk})$,
$|M|$ denotes the matrix obtained by taking the absolute value of each entry of
$M$, i.e., $[\,|M|\,]_{\alpha\sigma} = \big|[M]_{\alpha\sigma}\big|$.
The \textit{normalized boundary matrix} is then $\B_\kk \define D_\kk^{-1/2} B_\kk D_{\kk+1}^{1/2}$.
The symmetrically normalized, weighted Hodge Laplacian is defined by
\begin{equation*}
    \Lap_\kk \define \B_{\kk-1}^\transp \B_{\kk-1} + \B_\kk \B_\kk^\transp\,.
\end{equation*}
Throughout the remainder of this article, when referring to a variant of the Hodge Laplacian computed on a $\kk$-region $R$ without specifying the normalization or weighting, we write $\Lap(R)$. 
When $R=C_\kk$, we abbreviate this to $\Lap \define \Lap(C_\kk)$.

\subsection{The \texorpdfstring{$\kk$}{k}-HGLET}
The \emph{$\kk$-HGLET} is a generalization of the \emph{Hierarchical Graph Laplacian Eigen Transform} (HGLET)~\cite{IRION-SAITO-HGLETS} from functions on the nodes of a graph to functions on the $\kk$-simplices of a given simplicial complex~\cite{SAITO-SCHONSHECK-SHVARTS}. The HGLET itself can be viewed as a generalization of the Hierarchical Block Discrete Cosine Transform (HBDCT), which is obtained by creating a hierarchical bipartition of the signal domain and computing the DCT of the local signal supported on each subdomain. For the reader's convenience, we summarize the key points of the hierarchical bipartition tree of an input graph~$G$ and the original HGLET for node signals in Appendices~\appLetter{app:HBT} and~\appLetter{app:HGLET-GHWT}, respectively.

Let $\{\bphi^p_{k,l}\}$ denote the set of basis vectors in the $\kk$-HGLET dictionary, where $p$ is the level of the partition (with $p=0$ being the root), $k$ indexes the partition within the level, and $l$ indexes the elements within each partition in order of increasing frequency. Let $C^p_k$ be the $\kk$-region corresponding to the support of partition~$k$ at level~$p$, and let $n^p_k \define |C^p_k|$. Thus $C^0_0 = C_\kk$ and $n^0_0 = |C_\kk| \, =: \, n$.

To compute the transform, we first obtain the complete set of eigenvectors $\{\bphi^0_{0,l}\}_{l=0:n^0_0-1}$ of $\Lap = \Lap(C^0_0)$, ordered by nondecreasing eigenvalues. We then bipartition $C^0_0$ into two disjoint $\kk$-regions $C^1_0$ and $C^1_1$ using the Fiedler vector of $\Lap$. We note that: 1) other bipartitioning methods may be used; and 2) bipartitioning with the Fiedler vector in the $\kk$-region setting requires additional steps compared to the graph setting, due to its more intricate behavior (see~\cite{SAITO-SCHONSHECK-SHVARTS} for details).

The same procedure is applied recursively to $C^1_0$ and $C^1_1$, yielding eigenvectors $\{\bphi^1_{0,l}\}_{l=0:n^1_0-1}$ and $\{\bphi^1_{1,l}\}_{l=0:n^1_1-1}$. Since $n^1_0 + n^1_1 = n^0_0 = n$ and the supports of $\{\bphi^1_{0,l}\}$ and $\{\bphi^1_{1,l}\}$ are disjoint, these two sets are orthogonal. Their union $\{\bphi^1_{0,l}\} \cup \{\bphi^1_{1,l}\}$ forms an ONB for vectors in $X_\kk$. Continuing this process down the hierarchical bipartition tree produces an ONB at each level. If the tree terminates with regions containing single $\kk$-simplices, the final level coincides with the standard basis of $\Rf^n$.

Each level of the dictionary contains an ONB whose vectors have support roughly half the size of those at the previous level. There are approximately $(1.5)^n$ possible ONBs obtainable by selecting different covering sets of regions from the hierarchical bipartition tree. The computational cost of generating the entire dictionary is $\mathcal{O}(n^3)$. See~\cite{SAITO-SCHONSHECK-SHVARTS} for the complete algorithm to construct the $\kk$-HGLET on a given $C_\kk$ and for further implementation details.

\subsection{The \texorpdfstring{$\kk$}{k}-GHWT}
The \emph{$\kk$-GHWT} is a generalization of the \emph{Generalized Haar--Walsh Transform} (GHWT)~\cite{IRION-SAITO-GHWT} from functions on the nodes of a graph to functions on the $\kk$-simplices of a given simplicial complex~\cite{SAITO-SCHONSHECK-SHVARTS}. The GHWT itself can be viewed as a generalization of the Haar--Walsh wavelet packets~\cite[Sec.~8.1]{MALLAT-BOOK3}. As with the $\kk$-HGLET, the construction begins by generating a hierarchical bipartition tree of $C_\kk$.

The transform is computed in a bottom-up manner, starting from the finest level $p = \pmax$, where each region contains a single $\kk$-simplex represented by its indicator vector (i.e., one of the standard basis vectors of $\Rf^n$). These vectors are called \emph{scaling vectors} and are denoted $\{\bpsi^{\pmax}_{k,0}\}_{k=0:n-1}$. At the next level, $p = \pmax - 1$, we first assign a constant scaling vector supported on each region. For any region with two children in the partition tree, we form a \emph{Haar vector} by subtracting the scaling vector of the higher-index child from that of the lower-index child. This yields an orthonormal basis (ONB) $\{\bpsi^{\pmax-1}_{k,l}\}_{k=0:K-1,\, l=0:l(k)-1}$, where $K$ is the number of $\kk$-regions at level $\pmax-1$ and $l(k) \in \{1,2\}$, with each vector supported on at most two simplices.

At level $p = \pmax - 2$, we again compute scaling and Haar vectors as before. In addition, for any region containing three or more elements, we construct \emph{Walsh vectors} by adding and subtracting the Haar vectors of its child regions. This process is then applied recursively up the tree, producing the full dictionary. A complete description of the algorithm is given in~\cite{IRION-SAITO-GHWT} for the $\kk=0$ case and in~\cite{SAITO-SCHONSHECK-SHVARTS} for the general case $\kk > 0$.

Like the $\kk$-HGLET, each level of the $\kk$-GHWT dictionary forms an ONB, and the basis vectors at each level have support roughly half the size of those at the previous level. Moreover, the $\kk$-GHWT basis vectors share the same supports as the corresponding $\kk$-HGLET basis vectors, i.e., $\supp(\bphi^p_{k,l}) = \supp(\bpsi^p_{k,l})$ for all $p, k, l$. However, the computational cost of constructing the $\kk$-GHWT is only $\mathcal{O}(n \log n)$, in contrast to the $\mathcal{O}(n^3)$ cost of the $\kk$-HGLET.

\section{Multiscale Hodge Scattering Transform}

To discuss our version of the scattering transform on a given simplicial complex, it is convenient to define the \emph{scale parameter} $j \define \pmax - p$, where $p$ is the partition level introduced in the previous section. Thus, $j = 0$ corresponds to the finest partition $p = \pmax$, where each partition is a singleton whose basis vector is one of the standard basis vectors of $\Rf^n$, and $j = \jmax \define \pmax$ corresponds to the coarsest scale at the root level $p = 0$ (i.e., no partition) of the input simplicial complex.
Let the $\kk$-HGLET or $\kk$-GHWT dictionary vectors be arranged as
$\bPhi^{J} \define \left\{ \Phi^j \right\}_{j=0}^J$  where each $\Phi^j \in \Rf^{n \times n}$ is an ONB at scale $j$ whose \emph{rows} are the basis vectors. 
In general, the hierarchical bipartition scheme yields $\jmax + 1 \approx \log_2 n + 1$ distinct levels, but in practice features extracted at large $j$ are often
less descriptive~\cite{gao2019geometric}. Hence, we typically use the $J+1$
finest levels with $0 \leq J \leq \jmax$.

Let $\f \in X_\kk$ and write $[\f]_i$ for its value on simplex $i$.
For $q \in \N$, define \\
$\left|\f\right|^q \define \left(\left|[\f]_1\right|^q, \ldots, \left|[\f]_n\right|^q\right)^\transp \in \Rf^n$.
We compute the $q$th moments (up to some maximum $Q \in \N$) of the $0$th- and
and $1$st-order scattering coefficients:
\begin{equation}
\label{eq:s01}
\!\!\!\!\!  S^0(q) \f \define \frac{1}{n} \sum_{i=1}^n [\f]_i^q, \,\quad  S^1(q,j) \f \define \frac{1}{n} \sum_{i=1}^n \left[ \left| \Phi^j \f \right|^q \right]_i, \,\, 0 \leq j \leq J; 1 \leq q  \leq Q\,,
\end{equation}
and the $2$nd-order scattering coefficients:
\begin{equation}
\label{eq:s2}
S^2 \left( q,j,j'\right) \f \define  \frac{1}{n} \sum_{i=1}^n \left[ \left| \Phi^{j'} \left| \Phi^j \f \right| \right|^q \right]_i, 
\,\, 0 \leq j < j' \leq J, \,\, 1 \leq q  \leq Q\,.
\end{equation}
Higher-order scattering coefficients are defined similarly:
\begin{equation}
\label{eq:sm}
S^m \left( q,j^{(1)},\ldots,j^{(m)}\right) \f \define \frac{1}{n} \sum_{i=1}^n \left[ \left|\Phi^{j^{(m)}} \left| \Phi^{j^{(m-1)}} \left| \cdots \left| \Phi^{j^{(1)}} \f \right| \cdots \right| \right| \right|^q \right]_i,
\end{equation}
where $0 \leq j^{(1)} < \cdots < j^{(m)} \leq J$.
Because the number of features grow combinatorially with the order $m$,
it is uncommon to use orders beyond $m=2$ or $3$. Moreover, as our numerical
experiments indicate, higher moments ($q>4$) are typically less useful in practice due to instability~\cite{BRUNA-MALLAT, gao2019geometric, chew2022geometric}.

For clarity, define the layer operator
$\hat{S}^1(j,q) \f \define \left| \Phi^j \f \right|^q$,
which maps $\f \in \Rf^n$ to $\Rf_{\geq 0}^\ndim$, 
so that the $m$th-order feature in Eq.~\eqref{eq:sm} can be written succinctly as
\begin{equation}
\label{eq:nest}
S^m \left(q,j^{(1)},\ldots,j^{(m)} \right) \f 
= S^0(q) \hat{S}^1\left(1,j^{(m)}\right) \cdots \hat{S}^1 \left(1,j^{(1)} \right) \f
\end{equation}
This mirrors the architecture of a convolutional neural network with fixed
weights, as studied in~\cite{MALLAT-SCAT}. Unlike conventional feed-forward
networks, in which only the input to the final layer (i.e., the classification
head)---the features extracted at the end of the cascade of linear and nonlinear
transformations---is used for classification, 
we concatenate the features from every layer of the cascade for downstream analysis.
We refer to the process of extracting these features as the \emph{Multiscale Hodge Scattering Transform} (MHST).
When these deterministic features are used in conjunction with a separate
learnable model, e.g., support vector machine (SVM) or logistic regression model (LRM), we call the system a \emph{Multiscale Hodge Scattering Network} (MHSN). 

Figure~\ref{fig:MHSN-GP} illustrates our MHSN ($q=1$) with global pooling on a toy $2$-complex with nine elements.
The leftmost column shows $2$-GHWT bases at successive scales, with colors indicating the sign of the basis vector components (yellow = positive; purple = negative; green = 0). For $m=0$, $\hat{S}^0\f \in \Rf^9_{\geq 0}$ is computed by taking inner products of the input signal with the finest scale ($j=0$) $\delta$-function basis; the resulting coefficients are then averaged over all simplices---this averaging over the whole domain is what we call \emph{global pooling}---to produce a scalar $S^0 \f \in \Rnn$.
For $m=1$, $\hat{S}^1(j)\f \in \Rf^9_{\geq 0}$ is obtained by taking the inner products of $\hat{S}^0\f$ (identical to $\f$) with the scale-$j$ GHWT basis vectors, then applying pointwise modulus (absolute value). Global pooling yields the scattering coefficient $S^1(j)\f \in \Rnn$. Repeating this for each $j=0, 1, \ldots, J$, produces $S^1\f \in \Rf^{J+1}_{\geq 0}$.
For $m=2$, the same steps are applied to each $\hat{S}^1(j)\f$ with the scale-$j' (> j)$ GHWT basis vectors (illustrated for $j'=j+1$), iterating over all $j, j'$ with $0 \leq j < j' \leq J$, yielding $S^2\f \in \Rf^{J(J+1)/2}_{\geq 0}$.
This is repeated until $m$ reaches the user-specified highest order $M$ (in the figure, $M=4$ and $J=4$).
\begin{figure}[h]
 \includegraphics[width=\textwidth]{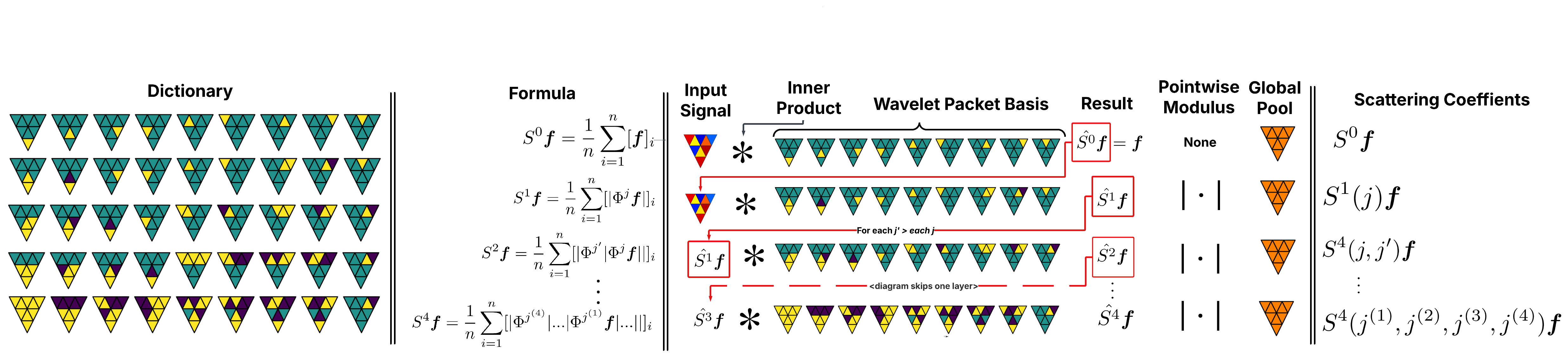}
 \caption{Schematic of the Multiscale Hodge Scattering Network for a $2$-complex with nine-elements, illustrating global pooling}
   \label{fig:MHSN-GP}
\end{figure}

\subsection{Local Pooling}\label{sec:local}

In general, correcting all moments up to $Q$ and of orders up to $M$ yields
$Q \sum_{m=0}^{M} \binom{J+1}{m}$ features under global pooling, since
the sums in \eqref{eq:s01}--\eqref{eq:sm} average over all $n$ simplices.
When permutation invariance is not required (i.e., all signals are defined on
a fixed complex with a known node ordering), we may omit these sums (\emph{no pooling}), resulting in $Q n \sum_{m=0}^{M} \binom{J+1}{m}$ features per signal.

Between these two extremes, we can perform \emph{local pooling} by averaging
over regions $C^j_k$ at a chosen scale $j$, instead of averaging over the entire
domain or not averaging at all. This produces a tuple of features (one per
region) rather than a single value as in the global pooling case. This idea is
related to windowed scattering transforms~\cite{chew2022geometric}, but here we
leverage the multiscale decomposition provided by the hierarchical partition
tree, avoiding an additional window-size parameter. In local pooling, we replace the normalization factor in the average by the region size $n^j_k$ (rather than
$n$), as formalized in Eq.~\eqref{eq:localstc}. The total number of features
is 
$Q \left( 1 + \sum_{m=0}^{M-1} \sum_{j=m}^J \binom{j}{m} K^j \right) = Q \left( 1 + \sum_{j=0}^J \left(\sum_{m=0}^{M-1} \binom{j}{m} \right) K^j \right)$,\\
where $K^j$ is the number of regions at scale $j$. Setting $K^j \equiv 1$
recovers the global-pooling count, whereas $K^j \equiv n$ recovers the
no-pooling count.

We denote these transforms as $S^m_j$, where the subscript $j$ indicates
the \emph{pooling scale} (the scale at which the final averaging is performed).
Thus, $S^m_\jmax = S^m$ corresponds to global pooling as in \cite{BRUNA-MALLAT, gao2019geometric}; $S^m_0$ denotes the transform without any pooling ($K^0=n$);
and $S^m_2$ denotes local pooling with regions at scale $j=2$.
In general,
\begin{equation}
\label{eq:localstc}
    S^m_j \left(q,j^{(1)},\ldots,j^{(m)}\right) \f \define \Biggl\{ \frac{1}{n^j_k} \sum_{i \in C^j_k} \left[ \left|\Phi^{j^{(m)}} \left| \Phi^{j^{(m-1)}} \left| \cdots \left| \Phi^{j^{(1)}} \f \right| \cdots \right| \right| \right|^q \right]_i \Biggr\} _{k=0}^{K^j-1} ,
\end{equation}
where $n^j_k = | C^j_k |$ and $C^j_k$ is the $k$th region at scale $j$ in the partition of the input $\kk$-simplices (i.e., partition level $p=\jmax-j$).
%

Figure~\ref{fig:MHSN-LP} illustrates the local-pooling version of our MHSN ($q=1$) using the same example as in Fig.~\ref{fig:MHSN-GP}.
The key difference between this variation and the global-pooling version (i.e., averaging over all simplices) is that averages are taken over a set of disjoint local regions that together cover the entire set of simplices. As a result, instead of producing a single feature at scale $j$, we obtain one feature for each local pool at that scale. The orange triangles indicate the $2$-regions in each local pool.
We construct these local pools by selecting one of the partitions of the $2$-simplices generated by the hierarchical bipartition tree used in computing the GHWT dictionary. See Appendix~\appLetter{app:HBT} and \cite{SAITO-SCHONSHECK-SHVARTS} for a detailed discussion of the partitioning scheme. In this figure, we use the partition defined by the support of the basis at scale $j=J-1=3$.
\begin{figure}[h]
 \includegraphics[width=\textwidth]{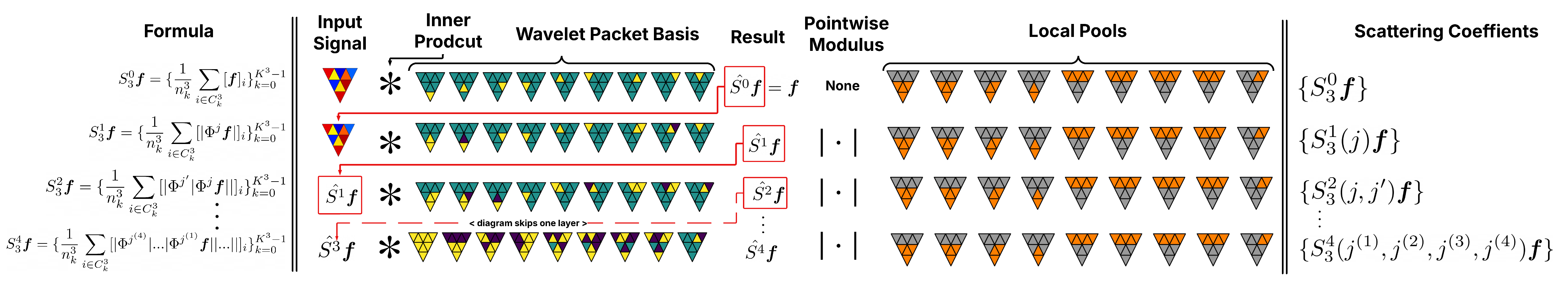}
 \caption{Schematic of the Multiscale Hodge Scattering Network for the same $2$-complex example as in Fig.~\ref{fig:MHSN-GP}, illustrating \emph{local} pooling}
 \label{fig:MHSN-LP}
\end{figure}

\section{Theoretical Analysis of Multiscale Hodge Scattering Transforms}
In this section, we establish the continuity, invariance, and equivalence
properties of the MHSNs that make them well suited for signal and domain
classification tasks. For notational convenience, we let $S \equiv S^1$ and
restrict our formal proofs to the first-order transform. However, since $S^m \f$
can be viewed as applying $S^1$ to the transform $S_0^{m-1}\f$, followed by
the pooling operation described in Eq.~\eqref{eq:nest}, all of the proofs
extend directly to the general $S^m$ case. 

First, we establish bounds on the coefficients generated by the multiscale basis.
Let $\W$ denote the matrix obtained by stacking a multiscale basis into a
$(J+1) n \times n$ matrix, where the $j$th $n\times n$ block corresponds to
the ONB at scale $j$.
We define a weighted inner product and the associated norm for signals
in $X_\kk$ by
\begin{equation*}
    \langle \f, \boldsymbol{g} \rangle_w \define \sum_{\sigma \in C_\kk} [\f]_\sigma [\bg]_\sigma [\bw]_\sigma \quad \text{and} \quad \|\f\|_w \define \sqrt{\langle \f, \f \rangle_w} .
\end{equation*}

In practice, we often take $\bw \equiv \bone$, but in some applications more
specialized weightings---such as those incorporating the volume of each face---may be advantageous.

\begin{remark}[Tight Frames]
  Let $\f \in X_\kk$, and $\| \cdot \|_w$ denote the $\ell^2$-norm with respect to
  the weight vector $\bw$. Then,
\begin{equation*}
    \|\f\|_w^2 \leq \|\W \f\|_w^2 = (J+1) \|\f\|^2_w\,.
\end{equation*}
\end{remark}

This follows immediately from the fact that $\W$ is the vertical concatenation of $J+1$ orthogonal matrices, and orthogonal transforms preserve the $\ell^2$-norm.
Although elementary, this property will play a key role in subsequent proofs.
We next show that the MHST is a non-expansive operator, enabling the use of
powerful tools from nonlinear contraction analysis~\cite{lohmiller1998contraction}, as illustrated in our numerical experiments.

\begin{theorem}[Non-expansive Operation]
  Let $S$ denote the MHST with global pooling, formed from the multiscale basis
  dictionary $\W$ as defined in \eqref{eq:s01}--\eqref{eq:sm}, with $q=1$,
  acting on the metric space $(X_\kk, \| \cdot \|_w)$.
  Then, for all $\f_1, \f_2 \in X_\kk$, 
\begin{equation*}
    \| S \f_1 - S \f_2 \|_w \leq \|\f_1 - \f_2\|_w\,.
\end{equation*}
Moreover, let $S_j$ denote the transforms with local pooling as described in Section~\ref{sec:local}. Then, for all $\ 0 \leq j \leq J$ and all $\f_1, \f_2 \in X_\kk$, 
\begin{equation*}
    \| S_j \f_1 - S_j \f_2 \|_w \leq \|\f_1 - \f_2\|_w
\end{equation*}
\end{theorem}

\begin{proof} We first show that each layer of the MHST is non-expansive.
  Since the full transform is a cascade of such layers, with pointwise modulus
  operations between them, the entire transform will also be non-expansive.

  Fix $j$ and consider the first-order coefficient $S^1(1,j)\f$.
  For $\f_1, \f_2 \in X_\kk$,
\begin{align*}
  \left| S^1\left(1,j \right) \f_1 - S^1\left(1,j \right)\f_2 \right|
  &= \left|
  \frac{1}{n} \sum_{i=1}^n \left[ \left| \Phi^j \f_1 \right| \right]_i
  - \frac{1}{n} \sum_{i=1}^n \left[ \left| \Phi^j \f_2 \right| \right]_i
  \right| \\
  &\leq \frac{1}{n} \sum_{i=1}^n \left|
  \left[ \left| \Phi^j \f_1 \right| \right]_i - \left[ \left| \Phi^j \f_2 \right| \right]_i
  \right| \\
  &\leq \frac{1}{n} \sum_{i=1}^n \left|
  \left[ \Phi^j \f_1 \right]_i - \left[ \Phi^j \f_2 \right]_i
  \right| \\
  &= \frac{1}{n} \sum_{i=1}^n \left|
  \left[ \Phi^j (\f_1 - \f_2) \right]_i
  \right| \\
  &= \frac{1}{n} \sum_{i=1}^n \left|
  \inner{\f_1 - \f_2}{\Phi^j(i,:)}
  \right| \\
  &\leq \frac{1}{n} \sum_{i=1}^n \left\|\f_1 - \f_2 \right\|_w \left\|\Phi^j(i,:)\right\|_w \quad \text{(Cauchy--Schwarz)} \\
  &= \frac{1}{n} \sum_{i=1}^n \left\|\f_1 - \f_2 \right\|_w \\
  &= \| \f_1 - \f_2 \|_w\,,
\end{align*}
where we have used the fact that each $\Phi^j$ is an ONB w.r.t.\ $\|\cdot\|_w$,
so $\|\Phi^j(i,:)\|_w = 1$ for all $i$.
Since this inequality holds for each $j$, taking the $w$-norm over the entire
collection of first-order features $S^1 \f$ preserves the bound.
Moreover, higher‑order features are obtained by applying $S^1$ to the previous
order's output (followed by pooling), so the same argument applies inductively.
Thus
\begin{equation*}
  \| S \f_1 - S \f_2 \|_w \leq \|\f_1 - \f_2\|_w\,. 
\end{equation*}
The proof for the local-pooling transforms $S_j$ follows identically, with the global average replaced by the average over each local region.
\end{proof}

Next, we show that our networks are invariant under group operations that
preserve the weighted inner product $\langle \cdot, \cdot, \rangle_w$,
such as permutations of the $\kk$-simplices. In other words, relabeling the
indices of the elements of $C_\kk$ does not affect the output of the transform.
For example, given a signal $\f$ defined on the triangular faces of a simplex,
the indexing of the triangles does not affect the globally-pooled transform, and
permuting this indexing produces the corresponding permutation of the non-pooled
signal. This result is analogous to Theorem~3 in~\cite{chew2022geometric} and
Proposition~4.1 in~\cite{zou2020graph}, but it applies to any $\kk$, rather than
only to the node case ($\kk=0$). 

\begin{theorem}[Invariance and Equivariance] 
  Let $\mathcal{G}$ be a group of transformations on the elements of $C_\kk$
  (e.g., permutations of the $\kk$-simplices). For any $\bxi \in \mathcal{G}$,
  let $V_{\bxi} : X_{\kk} \to X_{\kk}$ denote the operator
  $V_{\bxi} \f = \f \circ \bxi$, i.e.,
  $(V_{\bxi} \f)(\sigma) = \f\(\bxi(\sigma)\)$, induced by $\bxi$. 
  Let $S^{[\bxi]}$ denote the analogous transform on $C^{[\bxi]}_\kk$,
  the permuted version of $C_\kk$. Then, for $\bxi \in \mathcal{G}$,
  $\f \in X_\kk$, and $0 \leq j \leq J$, 
\begin{equation*}
    S^{[\bxi]} V_{\bxi} \f = S \f\,, \qquad S_j^{[\bxi]} V_{\bxi} \f = S_j \f\,.
\end{equation*}
\end{theorem}
\begin{proof}
  As in the previous proof, it suffices to show that the property holds for
  an arbitrary layer; since the full transform is a cascade of such layers,
  the result then follows for the entire transform.

  Let $\Lap_{\bxi} = V_{\bxi} \circ \Lap \circ V_{\bxi}^{-1}$, where $\Lap$ is
  the $\kk$-Laplacian on $C_\kk$. 
  If $\bphi$ is an eigenvector of $\Lap$ with $\Lap \bphi = \lambda \bphi$,
  then
  \begin{equation*}
    \Lap_{\bxi} V_{\bxi} \bphi
    = V_{\bxi} \circ \Lap \circ V_{\bxi}^{-1} \circ V_{\bxi} \bphi
    = V_{\bxi} \circ \Lap \bphi
    = \lambda\, V_{\bxi} \bphi .
  \end{equation*}
  Thus $V_{\bxi} \bphi$ is an eigenvector of $\Lap_{\bxi}$ with the same eigenvalue $\lambda$.

  Since $\mathcal{G}$ preserves the weighted inner product on $\W$, it follows
  that $\Phi_{\bxi} \define V_{\bxi} \Phi$ is an ONB for $\W_{\bxi}$, where the
  atoms of $\Phi_{\bxi}$ are eigenvectors of $\Lap_{\bxi}$.
  In particular, $\Phi_{\bxi} = V_{\bxi} \Phi V_{\bxi^{-1}}$.

  Now, for arbitrary input $f \in X_\kk$,
\begin{equation*}
  S^{[\bxi]} V_{\bxi} \f
  = \frac{1}{n} \sum_{i=1}^n \left| [\Phi_{\bxi}  V_{\bxi} \f]_i \right|
  = \frac{1}{n} \sum_{i=1}^n \left| [ V_{\bxi} \Phi \f ]_i \right|
  = \frac{1}{n} \sum_{i=1}^n \left| [\Phi \f]_i \right|
  = S \f
\end{equation*}

The proof for the local‑pooling transforms $S_j$ is identical, with the global
average replaced by the average over each local region.
\end{proof}

\section{Signal Classification}

We first demonstrate the effectiveness of our MHSNs with the article category
classification problem using the \emph{Science News} database~\cite{SCIENCE-NEWS, IRION-SAITO-MLSP16}. We do not claim SotA results for this problem, but rather use it to illustrate the advantages of analyzing signals via higher-dimensional simplices. We also show how locally-pooled networks can shatter problems in which traditional globally-pooled scattering networks fail to differentiate between classes.
The Science News dataset contains 1042 scientific
news articles classified into eight fields: Anthropology, Astronomy;
Behavioral Sciences; Earth Sciences; Life Sciences; Math/CS; Medicine; Physics.
Each article is tagged with keywords from a pool of 1133 words selected by the database curators. We determine a simplicial complex from these keywords by computing their \texttt{word2vec}~\cite{rong2014word2vec} embeddings based on
Google's publicly available pre-trained model~\cite{google-word2vec-site}.
We generate a symmetric $k$-nearest neighbor graph of the embedded words
and then generate $\kk$-simplices of the graph. 
Therefore, a $\kk$-simplex in this keyword graph corresponds to $\kk$-face, which represents a combination of $\kk+1$ words.

Next, we define representations of each article as a signal in each $X_\kk$
as follows. First, for $\kk=0$ (i.e., a node-valued signal), we define the
signal $\f_0$ to be one on the nodes representing their keywords and zero
elsewhere. For $\kk \geq 1$ we define the signal $\f_{\kk}$ to be the simplex-wise
average of the $\f_0$ signal. That is,
\begin{equation*}
[\f_0]_i = \begin{cases}
1 &\text{if keyword $i$ occurs}\\
0 &\text{otherwise}
\end{cases} ;
\quad
[\f_\kk]_i = \frac{1}{\kk+1} \sum_{\substack{l \in V(\sigma_i) \\\sigma_i \in C_\kk}} [\f_0]_l ,
\end{equation*}
where $V(\sigma_i)$ represents the set of nodes forming the $i$th simplex $\sigma_i \in C_\kk$. Note that these signals are highly localized since the keywords are connected through a symmetrized $k$NN graph, and the higher-order signals are built from the adjacency of the resulting complex. To showcase the robustness of our approach, we report results using both $k=5$ and $k=10$ nearest neighbor graphs.          

Tables~\ref{tab:5-NN} and \ref{tab:10-NN} compare the performance of our proposed methods with the other simpler methods, i.e., the raw expansion coefficients of the signals relative to the standard ONBs (Delta; Fourier) or the dictionaries (Diffusion; HGLET; GHWT).
The parameters for the feature computations were set as $(J, M, Q)  = (5, 3, 4)$.
For each $\kk$, we performed the five-fold cross-validation, i.e., we randomly split these 1042 signals into 80\% training and 20\%
test sets and repeat the experiment 5 times with different train/test splits. 
In every case we use the $\ell^2$-regularized LRM provided by \texttt{scikit-learn}~\cite{scikit-learn} without any additional hyperparameter tuning to compute the classification.

Several observations are in order. First, the traditional, globally-pooled scattering networks mostly fail on this task regardless of the wavelet dictionary employed. Since the number of nonzero entries in each signal is similar and therefore the $\ell^1$-norms are also similar, global-pooling schemes fail to capture the keyword information (i.e., indices of nonzero entries) in a way that differentiates between the classes and consequently do not produce statistically significant results. The non-pooled features often provide the highest performance, which is not surprising since there are many more features and learnable parameters than the networks with pooling. However, the locally-pooled features almost always perform on par with the non-pooled features. For both the 5 and 10 nearest neighbor graphs, the best overall results are achieved by the $\kk$, which has the largest number of elements. Similarly, the 10-nearest neighbor graph performs better than the 5-nearest neighbor graphs at the cost of larger $n$.  

We also observe that the networks based on $\kk$-HGLET and $\kk$-GHWT generally outperform those based on Diffusion Wavelets. This is likely due to the highly localized and piecewise constant nature of the input signals, which are well-approximated by these dictionaries \cite{SAITO-SCHONSHECK-SHVARTS}. In the next section, where the signals are not localized, we do not observe this difference.

\begin{table*}[ht]
\centering
\resizebox{\textwidth}{!}{%
\begin{tabular}{c|c|cc|ccc|cccc|cccc}
\hline
Knn = 5    & $n$    & \multicolumn{1}{c|}{Delta} & Fourier & \multicolumn{3}{c|}{Diffusion}                                              & \multicolumn{4}{c|}{HGLET}                                                                       & \multicolumn{4}{c}{GHWT}                                                                         \\ \hline
           &      & \multicolumn{1}{c|}{Basis} & Basis   & \multicolumn{1}{c|}{Dict.} & \multicolumn{1}{c|}{GP} & NP         & \multicolumn{1}{c|}{Dict.} & \multicolumn{1}{c|}{GP} & \multicolumn{1}{c|}{LP} & NP              & \multicolumn{1}{c|}{Dict.} & \multicolumn{1}{c|}{GP} & \multicolumn{1}{c|}{LP} & NP              \\ \hline
$\kk$=0 & 1133 & 33.971                     & 33.971  & \uline{\textbf{86.603}}                     & 31.579                       & \uline{\textbf{86.603}} & 81.818                     & 31.579                  & \uline{\textbf{86.603}}         & \uline{\textbf{86.603}} & 80.861                     & 31.579                  & 85.646                  & 86.124          \\
$\kk$=1 & 3273 & 55.502                     & 78.947  & 85.646                     & 31.579                       & 85.646          & \textbf{86.124}            & 31.579                  & \textbf{86.124}         & 85.646          & 85.603                     & 31.579                  & 85.646                  & 85.646          \\
$\kk$=2 & 1294 & 55.502                     & 49.761  & 83.732                     & 31.579                       & 83.732          & 83.254                     & 31.579                  & \textbf{84.211}         & 83.732          & 83.732                     & 31.579                  & 83.254                  & 83.254          \\
$\kk$=3 & 227  & 31.579                     & 31.579  & \textbf{78.947}                     & 31.579                       & \textbf{78.947}          & 51.675                     & 31.579                  & 78.469                  & \textbf{78.947} & 51.196                     & 31.579                  & 78.469                  & \textbf{78.947} \\
$\kk$=4 & 16   & 31.579                     & 31.579  & \textbf{55.981}            & 31.100                         & \textbf{55.981} & 32.057                     & 31.100                    & \textbf{55.981}         & \textbf{55.981} & 32.057                     & 37.799                  & 55.502                  & 54.067          \\ \hline
\end{tabular}
}
\vspace{0.1em}
\caption{Article category classification accuracy for $5$-NN graph of the Science News dataset for different simplex degrees. Dict.\ implies that the SVM is trained solely on the dictionary coefficients while GP, LP, NP imply scattering networks with global, local, and no pooling, respectively. The best performer for each $\kk$ is indicated in bold while the underlined bold numbers are the best among all $\kk$'s.}
\label{tab:5-NN}
\end{table*}

\begin{table*}[ht]
\centering
\resizebox{\textwidth}{!}{%
\begin{tabular}{c|c|cc|ccc|cccc|cccc}
\hline
Knn = 10                        & $n$    & \multicolumn{1}{c|}{Delta} & Fourier & \multicolumn{3}{c|}{Diffusion}                                              & \multicolumn{4}{c|}{HGLET}                                                                       & \multicolumn{4}{c}{GHWT}                                                                         \\ \hline
\multicolumn{1}{c|}{}           &      & \multicolumn{1}{c|}{Basis} & Basis   & \multicolumn{1}{c|}{Dict.} & \multicolumn{1}{c|}{GP} & NP              & \multicolumn{1}{c|}{Dict.} & \multicolumn{1}{c|}{GP} & \multicolumn{1}{c|}{LP} & NP              & \multicolumn{1}{c|}{Dict.} & \multicolumn{1}{c|}{GP} & \multicolumn{1}{c|}{LP} & NP              \\ \hline
\multicolumn{1}{c|}{$\kk$=0} & 1133 & 35.238                     & 35.238  & 60.952                     & 32.381                       & 87.619          & 81.905                     & 32.381                  & \textbf{88.571}         & 87.619          & 80.952                     & 32.381                  & 87.619                  & 87.619          \\
\multicolumn{1}{c|}{$\kk$=1} & 6890 & 81.905                     & 81.905  & 86.667                     & 32.381                       & 86.667          & 85.714                     & 32.381                  & \uline{\textbf{89.524}}         & 86.667          & 85.714                     & 32.381                  & \uline{\textbf{89.524}}         & \uline{\textbf{89.524}} \\
\multicolumn{1}{c|}{$\kk$=2} & 7243 & 76.19                      & 76.19   & 86.667                     & 32.381                       & 88.571          & 85.714                     & 32.381                  & 88.571                  & 88.571          & 88.571                     & 32.381                  & \uline{\textbf{89.524}}         & 88.571          \\
\multicolumn{1}{c|}{$\kk$=3} & 4179 & 69.524                     & 69.524  & 74.286                     & 33.333                       & \textbf{86.667} & \textbf{86.667}            & 33.333                  & \textbf{86.667}         & \textbf{86.667} & \textbf{86.667}            & 33.333                  & \textbf{86.667}         & \textbf{86.667} \\
\multicolumn{1}{c|}{$\kk$=4} & 1740 & 45.714                     & 45.714  & 68.571                     & 35.238                       & \textbf{81.905} & 73.333                     & 35.238                  & \textbf{81.905}         & \textbf{81.905} & \textbf{81.905}            & 33.333                  & \textbf{81.905}         & \textbf{81.905} \\
\multicolumn{1}{c|}{$\kk$=5} & 560  & 33.333                     & 33.333  & 39.048                     & 34.286                       & \textbf{73.333} & 60.952                     & 33.333                  & \textbf{73.333}         & \textbf{73.333} & 60.952                     & 34.286                  & \textbf{73.333}         & \textbf{73.333} \\
\multicolumn{1}{c|}{$\kk$=6} & 98   & 32.381                     & 32.381  & 32.381                     & 34.286                       & \textbf{62.857} & 39.048                     & 35.238                  & \textbf{62.857}         & \textbf{62.857} & \textbf{62.857}            & 35.238                  & \textbf{62.857}         & 60.952          \\ \hline
\end{tabular}
}
\caption{Article category classification accuracy for $10$-NN graph of the Science News dataset for different simplex degrees. Dict.\ implies that the SVM is trained solely on the dictionary coefficients while GP, LP, NP imply scattering networks with global, local, and no pooling, respectively. The best performer for each $\kk$ is indicated in bold while the underlined bold numbers are the best among all $\kk$'s. }
\label{tab:10-NN}
\end{table*}

\section{Domain Classification}
\label{sec:domclass}
Another vital application of geometric scattering networks and graph neural networks (GNNs) is graph (and simplex) classification. Broadly speaking, this problem consists of predicting a label of a social or chemical graph based on a training set of similar graphs with different configurations (i.e., different numbers of nodes and edges). For example, in the COLLAB dataset~\cite{yanardag2015deep, patania2017shape}, each graph represents a network of coauthors of a scientific paper.
Since the size of the graphs varies greatly within these datasets, we employ only the global-pooling version of our MHSN, akin to the previous efforts reported in \cite{gao2019geometric,chew2022geometric}, which were based on geometric scattering methods.

We compute permutation-invariant input features based only on
topological information obtainable from the boundary matrices. Since many of the
graphs are too small to contain high-degree simplices, we only consider node and
edge-based features and transforms. Following the methodology developed in
\cite{gao2019geometric}, we set $(J, M, Q)  = (4, 2, 4)$.
For the node signals, we first compute
the eccentricity and clustering coefficient \cite[Sec.\ 1.2]{HARRIS-ETAL} of
each node.
For each node signal, the number of parameters (MHST coefficients) are
64 via the formula $Q \sum_{k=0}^{M} \binom{J+1}{k}$, hence 128 parameters
after concatenating them.
For the edge signals, we use the average of the eccentricities of the head and tail nodes of each edge and the number of non-zero off-diagonal terms in the combinatorial Hodge-Laplacian (each such term corresponds to a $1$-adjacent edge \cite[Sec.\ 4.1]{SHVARTS-PHD}).


For each domain classification problem we train three models:
1) using 128 node features; 2) using 128 edge features;
and 3) using 256 combined features.
We then employ a simple SVM with Gaussian radial basis functions to classify
the features. Moreover, we tune the hyperparameters controlling the strength
of the $\ell^2$-regularization and the kernel coefficients via the
cross-validation scheme presented in \cite{gao2019geometric} using the
same search space and validation indexes. 

We compare these results with those obtained by the geometric scattering network (with Diffusion Wavelets) using SVM (GS-SVM) as well as several popular GNN models including the graph convolution network (GCN) \cite{kipf2017semi}, universal graph transform (UGT) \cite{nguyen2022universal}, dynamic graph CNN (DGCNN) \cite{wang2018dynamic}, graph attention network (GAT) \cite{velickovic2018graph}, and graph feature network (GFN) \cite{chen2019powerful}. For each competing method, we reproduce the results in the significant figures reported in their original publications; we report to 2 decimal places for our methods. More information on the benchmark datasets can be found in \ref{sec:datasets}. We remark that, as of this article's writing, this collection of networks achieves SotA results on these datasets according to the Papers with Code Leaderboards \cite{paperwithcode}. Further details on these datasets and their associated classification problems are presented in ~\ref{sec:datasets} and the references therein. 

\begin{table*}[ht]
\centering
\resizebox{\textwidth}{!}{%
\begin{hide}
\begin{tabular}{c|ccc|ccc|cccccc}
\hline
\multicolumn{1}{c|}{\textbf{}} &
& \multicolumn{3}{c|}{HGLET+SVM}                                                 & \multicolumn{3}{c|}{GHWT+SVM}
& \multicolumn{1}{c|}{\multirow{2}{*}{GS-SVM}}
& \multicolumn{1}{c|}{\multirow{2}{*}{GCN}}
& \multicolumn{1}{c|}{\multirow{2}{*}{UGT}}
& \multicolumn{1}{c|}{\multirow{2}{*}{DGCNN}}
& \multicolumn{1}{c|}{\multirow{2}{*}{GAT}}
& \multirow{2}{*}{GFN} \\ \cline{1-6}

\multicolumn{1}{c|}{Graph \ Feature Type}
& \multicolumn{1}{c|}{Node} & \multicolumn{1}{c|}{Edge} & \multicolumn{1}{c|}{Combo} & \multicolumn{1}{c|}{Node} & \multicolumn{1}{c|}{Edge} & \multicolumn{1}{c|}{Combo} & \multicolumn{1}{c|}{} & \multicolumn{1}{c|}{} & \multicolumn{1}{c|}{}  & \\ \hline
COLLAB   & 70.84 & 78.34 & 80.39 & 68.94 & 78.34 & 78.34
& 79.94  & 79.00 & 77.84 & 73.76 & 75.80 & \textbf{81.50}    \\
DD       & 60.67 & 68.73 & 72.71 & 72.20 & 65.76 & 69.32
         & -     & -     & \textbf{80.23} & 79.37 & - & 79.37 \\
IMDB-B   & 72.70 & 70.60 & 73.10 & 62.50 & 61.60 & 61.90
         & 71.20 & 74.00 & \textbf{77.04} & 70.03 & 70.50 & 73.40 \\
IMDB-M   & 44.40 & 47.13 & 49.68 & 44.33 & 39.07 & 39.73
         & 48.73 & 51.90 & \textbf{53.60} & 47.83 & 47.8 & 51.80 \\
MUTAG    & 85.78 & 86.31 & 85.78 & 81.58 & 76.84 & 79.95
         & 83.50 & 85.60 & 80.23 & 79.37 & \textbf{89.40} & 85.83 \\
PROTEINS & 73.57 & 73.04 & 75.35 & 71.52 & 66.70 & 71.61
         & 74.11 & 76.00 & \textbf{78.53} & 75.54 & 74.70 & 76.46 \\
PTC      & 62.85 & 67.71 & 68.28 & 52.29 & 55.71 & 51.71
         & 63.94  & 64.20 & \textbf{69.63} & 58.59 & 66.70 & 66.60              \\ \hline
\end{tabular}
\end{hide}
\begin{tabular}{c|ccc|ccc|c|c|c|c|c|c}
\hline
\textbf{} & \multicolumn{3}{c|}{HGLET+SVM} & \multicolumn{3}{c|}{GHWT+SVM}
& \multirow{2}{*}{GS-SVM} & \multirow{2}{*}{GCN} & \multirow{2}{*}{UGT} &
\multirow{2}{*}{DGCNN} & \multirow{2}{*}{GAT} & \multirow{2}{*}{GFN} \\
\cline{2-7}
Graph    & Node & Edge & Combo & Node & Edge & Combo & & & & & & \\ \hline
COLLAB   & 70.84 & 78.34 & 80.39 & 68.94 & 78.34 & 78.34 & 79.94 & 79.00 & 77.84 & 73.76 & 75.80 & \textbf{81.50} \\
DD       & 60.67 & 68.73 & 72.71 & 72.20 & 65.76 & 69.32 & -     & -     & \textbf{80.23} & 79.37 & -     & 79.37 \\
IMDB-B   & 72.70 & 70.60 & 73.10 & 62.50 & 61.60 & 61.90 & 71.20 & 74.00 & \textbf{77.04} & 70.03 & 70.50 & 73.40 \\
IMDB-M   & 44.40 & 47.13 & 49.68 & 44.33 & 39.07 & 39.73 & 48.73 & 51.90 & \textbf{53.60} & 47.83 & 47.80 & 51.80 \\
MUTAG    & 85.78 & 86.31 & 85.78 & 81.58 & 76.84 & 79.95 & 83.50 & 85.60 & 80.23 & 79.37 & \textbf{89.40} & 85.83 \\
PROTEINS & 73.57 & 73.04 & 75.35 & 71.52 & 66.70 & 71.61 & 74.11 & 76.00 & \textbf{78.53} & 75.54 & 74.70 & 76.46 \\
PTC      & 62.85 & 67.71 & 68.28 & 52.29 & 55.71 & 51.71 & 63.94 & 64.20 & \textbf{69.63} & 58.59 & 66.70 & 66.60 \\ \hline
\end{tabular}
}
\caption{Graph classification accuracy on seven datasets. The best performer for each dataset is indicated in bold. }
\label{tab:graphclassification}
\end{table*}

Although our MHSNs do not achieve SotA results on these datasets, they are very competitive with \emph{only a small fraction of the learnable parameters}. Moreover, the number of learnable parameters in our models is not tied to the graph size and depends only on the order of the scattering $M$ and the number of moments {$Q$} computed. For example, Table~\ref{tab:parameters} compares our $\kk$-HGLET-based MHSN with the UGT and the GFN, which are the SotA methods for various graph classification problems. These methods each require more than half a million parameters for some cases (867K for UGT) to achieve results similar to ours, requiring only 256 parameters to learn. As a result, our MHSNs can be implemented and trained on a consumer-level laptop, whereas many of these competing GNNs require specialized hardware. 

\begin{table*}[ht]
\centering
\resizebox{\textwidth}{!}{%
\begin{tabular}{c|cc|cc|cc}
\hline
\multicolumn{1}{l|}{} & \multicolumn{2}{c|}{$\kk$-HGLET + SVM}                   & \multicolumn{2}{c|}{UGT}
& \multicolumn{2}{c}{GFN}                                       \\ \hline
Graph                 & \multicolumn{1}{l|}{Accuracy} & \multicolumn{1}{l|}{\# Parameters} & \multicolumn{1}{l|}{Accuracy} & \multicolumn{1}{l|}{\# Parameters} & \multicolumn{1}{l|}{Accuracy} & \multicolumn{1}{l}{\# Parameters} \\ \hline
COLLAB                & 80.39                        & 256                           & 77.84                         & 866,746                        & 81.50                          &         68,754                   \\
DD                    & 72.71                         & 256                           & 80.23                         & 76,928                         & 79.37                         &            68,754                   \\
IMDB-B                & 73.10                          & 256                           & 77.04                         & 55,508                         & 73.40                          &        68,754                       \\
IMDB-M                & 49.68                         & 256                           & 53.60                          & 48,698                         & 51.80                          &          68,818                     \\
MUTAG                 & 85.78                         & 256                           & 80.23                         & 4,178                          & 85.83                         &          65,618                   \\
PROTEINS              & 75.35                         & 256                           & 78.53                         & 1,878                          & 76.46                         &        65,618                       \\
PTC                   & 68.28                         & 256                           & 69.63                          & 12,038                         & 66.60                             & 65,618        \\ \hline
\end{tabular}
}
\caption{Comparison of MHSN and the SotA graph classification networks in accuracy and number of learnable parameters}
\label{tab:parameters}
\end{table*}

\section{Molecular Dynamics}

Our MHSNs can also be used for regression problems where the goal is to predict a continuous property of a simplicial complex (or simply a graph) based on a set of observations of the complex under various conditions. Therefore, they are quite suitable for learning molecular dynamics, particularly the potential energy surface of a molecule, given a few registrations of the molecule and its energies. The Revised Molecular Dynamics 17 (rMD17 dataset) \cite{bowman2022md17} contains 100,000 structures and associated energies of various molecules. However, these structures are taken from a molecular dynamics simulation, i.e., time series data, which is not independent and identically distributed. To overcome this, instead of using the entire dataset, we use five sets of molecule snapshots and the associated potential energies. Each of these sets consists of 1,000 snapshot/energy pairs and is grouped into 800 training and 200 test samples selected by the authors of the dataset \cite{bowman2022md17}. 

We extract a rich set of features for each structure (i.e., a pose or conformation of a molecule) using our MHSNs (without pooling) and then employ a support vector regression (SVR) method with Gaussian radial basis functions to approximate and predict the energy.  More specifically, for each molecule, we first compute the average position of each atom across the training set. Then, using these positions, we create a $k$NN-graph (with $k=5$) as a template simplicial complex.
Note that by using this simplicial complex, rather than the molecular-bond graph, we can better capture the geometric information in the pose of the molecule as detailed in \cite{schutt2017schnet, schutt2023schnetpack}. 
Unlike the domain classification problems in Section~\ref{sec:domclass},
the geometry of the simplicial complex is fixed, so rather than using its geometrically-invariant descriptors, we need to begin
with signals that encode the position information of molecules for each $\kk$ which we wish to use in the MSHN.

First we compute the Euclidean distance matrix of (i.e., the Gram matrix of the point-cloud, measured in the Euclidean distance) the node coordinates of each snapshot and assign the corresponding column vector of the distance matrix as its node features. This generates a number-of-atoms-dimensional node signal for each node, for each snapshot.

For an edge signal, we extract edge lengths from the above distance matrix and create a diagonal matrix of edge lengths. Then, we assign the corresponding column vector of this diagonal matrix as its edge features. As with the node-based signal, this gives us number-of-edges-dimensional signal to input into our MHSN. We repeat this process using area and volume of the $\kk = 2$ and $\kk = 3$-simplices for the corresponding input singals of $\kk$-simplices. 

We use the setting $(J, M, Q)  = (4, 2, 3)$ to generate the scattering features and use SVR on our final layer output. Finally, we tune the hyperparameters in using the same cross validation scheme and settings in \ref{sec:domclass}, which was orignally proposed in \cite{gao2019geometric}. 

Table~\ref{tab:MD} shows our results for aspirin (21 atoms) and paracetamol (20 atoms) molecules. We compare our MHSNs with several SotA GNN approaches designed specifically for processing molecular dynamics, including SchNet \cite{schutt2017schnet}, PaiNN \cite{schutt2021equivariant}, and two variants of Special Orthogonal Networks (SO3Nets) \cite{batzner20223, schutt2023schnetpack}. 
We report both the mean absolute error (MAE) and root mean square error (RMSE) of the energy prediction, which are the standard metrics in the computational chemistry literature. Our MHSNs perform competitively with these approaches, while employing roughly 1\% as many learnable parameters as competing methods. 
Additionally, we observe the edge and triangle-based analyses outperform the node-based analysis. This demonstrates that higher-dimensional simplex analysis can be more powerful than node-only approaches, even in cases where the underlying molecular graph may not have many higher-dimensional structures. We also observe that a concatenation of these features, which we refer to in Table~\ref{tab:MD} as `Combo', outperforms any single $\kk$-based analysis, just as we observed in the experiments in Section \ref{sec:domclass}.

\begin{table}[ht]
\resizebox{\textwidth}{!}{%
\begin{tabular}{cccccccccccccccccccc}
\hline
\multicolumn{20}{|c|}{\textbf{Aspirin}}                                                                                                                                                                                                                                                                                                                                                                                                                                                                                                                                                                                                               \\ \hline
\multicolumn{1}{c|}{}             & \multicolumn{5}{c|}{\textbf{Diff+SVR}}                                                                                          & \multicolumn{5}{c|}{\textbf{HGLET+SVR}}                                                                                                 & \multicolumn{5}{c|}{\textbf{GHWT+SVR}}                                                                                 & \multicolumn{1}{c|}{\multirow{2}{*}{\textbf{SchNet}}} & \multicolumn{1}{c|}{\multirow{2}{*}{\textbf{PaiNN}}} & \multicolumn{1}{c|}{\multirow{2}{*}{\textbf{SO3Net I}}} & \multirow{2}{*}{\textbf{SO3Net II}} \\ \cline{1-16}
\multicolumn{1}{c|}{Feat}         & Node                 & Edge                 & Triangle             & Pyramid              & \multicolumn{1}{c|}{Combo}          & Node                 & Edge                 & Triangle             & Pyramid                      & \multicolumn{1}{c|}{Combo}          & Node                 & Edge                 & Triangle             & Pyramid              & \multicolumn{1}{c|}{Combo} & \multicolumn{1}{c|}{}                                 & \multicolumn{1}{c|}{}                                & \multicolumn{1}{c|}{}                                   &                                     \\ \hline
\multicolumn{1}{c|}{MAE}          & 4.856                & 3.132                & 2.93                 & 3.857                & \multicolumn{1}{c|}{\textbf{2.468}} & 4.884                & 3.135                & 3.016                & 3.876                        & \multicolumn{1}{c|}{2.651}          & 4.928                & 3.075                & 3.01                 & 3.842                & \multicolumn{1}{c|}{2.804} & \multicolumn{1}{c|}{13.5}                             & \multicolumn{1}{c|}{3.8}                             & \multicolumn{1}{c|}{3.8}                                & \textbf{2.6}                        \\
\multicolumn{1}{c|}{RMSE}         & 6.181                & 4.144                & 3.784                & 4.937                & \multicolumn{1}{c|}{\textbf{3.226}} & 6.215                & 4.129                & 3.905                & 4.976                        & \multicolumn{1}{c|}{3.473}          & 6.213                & 4.123                & 3.883                & 4.964                & \multicolumn{1}{c|}{3.613} & \multicolumn{1}{c|}{18.3}                             & \multicolumn{1}{c|}{5.9}                             & \multicolumn{1}{c|}{5.7}                                & \textbf{3.8}                        \\
\multicolumn{1}{c|}{\# Paramters} & 924                  & 3784                 & 6512                 & 5764                 & \multicolumn{1}{c|}{16984}          & 924                  & 3784                 & 6512                 & 5764                         & \multicolumn{1}{c|}{16984}          & 924                  & 6512                 & 6512                 & 5764                 & \multicolumn{1}{c|}{16984} & \multicolumn{1}{c|}{432k}                             & \multicolumn{1}{c|}{341k}                            & \multicolumn{1}{c|}{283k}                               & \textbf{341k}                       \\ \hline
\multicolumn{1}{l}{}              & \multicolumn{1}{l}{} & \multicolumn{1}{l}{} & \multicolumn{1}{l}{} & \multicolumn{1}{l}{} & \multicolumn{1}{l}{}                & \multicolumn{1}{l}{} & \multicolumn{1}{l}{} & \multicolumn{1}{l}{} & \multicolumn{1}{l}{}         & \multicolumn{1}{l}{}                & \multicolumn{1}{l}{} & \multicolumn{1}{l}{} & \multicolumn{1}{l}{} & \multicolumn{1}{l}{} & \multicolumn{1}{l}{}       & \multicolumn{1}{l}{}                                  & \multicolumn{1}{l}{}                                 & \multicolumn{1}{l}{}                                    & \multicolumn{1}{l}{}                \\
\multicolumn{1}{l}{}              & \multicolumn{1}{l}{} & \multicolumn{1}{l}{} & \multicolumn{1}{l}{} & \multicolumn{1}{l}{} & \multicolumn{1}{l}{}                & \multicolumn{1}{l}{} & \multicolumn{1}{l}{} & \multicolumn{1}{l}{} & \multicolumn{1}{l}{}         & \multicolumn{1}{l}{}                & \multicolumn{1}{l}{} & \multicolumn{1}{l}{} & \multicolumn{1}{l}{} & \multicolumn{1}{l}{} & \multicolumn{1}{l}{}       & \multicolumn{1}{l}{}                                  & \multicolumn{1}{l}{}                                 & \multicolumn{1}{l}{}                                    & \multicolumn{1}{l}{}                \\ \hline
\multicolumn{20}{|c|}{\textbf{Paracetamol}}                                                                                                                                                                                                                                                                                                                                                                                                                                                                                                                                                                                                           \\ \hline
\multicolumn{1}{c|}{}             & \multicolumn{5}{c|}{\textbf{Diff+SVR}}                                                                                          & \multicolumn{5}{c|}{\textbf{HGLET+SVR}}                                                                                                 & \multicolumn{5}{c|}{\textbf{GHWT+SVR}}                                                                                 & \multicolumn{1}{c|}{\multirow{2}{*}{\textbf{SchNet}}} & \multicolumn{1}{c|}{\multirow{2}{*}{\textbf{PaiNN}}} & \multicolumn{1}{c|}{\multirow{2}{*}{\textbf{SO3Net I}}} & \multirow{2}{*}{\textbf{SO3Net II}} \\ \cline{1-16}
\multicolumn{1}{c|}{Feat}         & Node                 & Edge                 & Triangle             & Pyramid              & \multicolumn{1}{c|}{Combo}          & Node                 & Edge                 & Triangle             & \multicolumn{1}{c|}{Pyramid} & \multicolumn{1}{c|}{Combo}          & Node                 & Edge                 & Trianlge             & Pyramid              & \multicolumn{1}{c|}{Combo} & \multicolumn{1}{c|}{}                                 & \multicolumn{1}{c|}{}                                & \multicolumn{1}{c|}{}                                   &                                     \\ \hline
\multicolumn{1}{c|}{MAE}          & 4.609                & 2.715                & 2.77                 & 3.947                & \multicolumn{1}{c|}{2.235}          & 4.723                & 2.643                & 2.896                & 3.995                        & \multicolumn{1}{c|}{\textbf{2.302}} & 4.748                & 2.624                & 2.843                & 3.928                & \multicolumn{1}{c|}{2.309} & \multicolumn{1}{c|}{8.4}                              & \multicolumn{1}{c|}{2.1}                             & \multicolumn{1}{c|}{2.2}                                & \textbf{1.4}                        \\
\multicolumn{1}{c|}{RMSE}         & 5.86                 & 3.418                & 3.596                & 4.961                & \multicolumn{1}{c|}{2.83}           & 5.964                & 3.338                & 3.709                & 4.99                         & \multicolumn{1}{c|}{\textbf{2.877}} & 5.961                & 3.299                & 3.679                & 4.916                & \multicolumn{1}{c|}{2.932} & \multicolumn{1}{c|}{11.2}                             & \multicolumn{1}{c|}{2.9}                             & \multicolumn{1}{c|}{3}                                  & \textbf{1.9}                        \\
\multicolumn{1}{c|}{\# Paramters} & 880                  & 3564                 & 5808                 & 4752                 & \multicolumn{1}{c|}{15004}          & 880                  & 3564                 & 5808                 & 4752                         & \multicolumn{1}{c|}{15004}          & 880                  & 3564                 & 5808                 & 4752                 & \multicolumn{1}{c|}{15004} & \multicolumn{1}{c|}{432k}                             & \multicolumn{1}{c|}{341k}                            & \multicolumn{1}{c|}{283k}                               & \textbf{341k}                       \\ \hline
\end{tabular}
}

\caption{Comparison of the performance of our MHSNs and the other SotA GNNs for potential energy prediction. We report the accuracy via MAE and RMSE as well as the number of trainable parameters in each network.}
\label{tab:MD}
\end{table}

\section{Conclusion}

In this article, we proposed the \emph{Multiscale Hodge Scattering Transforms/Networks} (MHSTs/MHSNs) for robust feature extraction from signals on simplicial complexes, applicable to both classification and regression tasks. Our approach fully leverages our multiscale basis dictionaries on simplicial complexes, namely the $\kk$-HGLET and $\kk$-GHWT dictionaries. The proposed MHSTs/MHSNs support multiple pooling options for the scattering transform outputs---no-pooling; local-pooling; and global-pooling---allowing them to address a wide range of problems, from classification of signals defined on simplicial complexes, to domain classification (i.e., classification of the simplicial complexes themselves), to regression of potential energies in molecular dynamics. We demonstrated that MHSNs achieve performance comparable to SotA GNNs while reducing the number of learnable parameters by up to two orders of magnitude. We attribute this success to the structure and organization of our multiscale basis dictionaries, which are naturally arranged by scales and locations, making them particularly well suited for generating scattering transform coefficients. 


Future work will focus on interpreting MHST coefficients identified as important by classification methods such as the LRMs. Due to the nonlinearities in MHSTs, mapping these coefficients back to interpretable features in the primal (original) domain is challenging. To address this, we plan to investigate the optimization method proposed recently in \cite{SAITO-WEBER-SSP25}, which synthesizes input patterns that maximize a model's predicted probability for a target class, thereby revealing the meaning of selected scattering coefficients.
In a related direction, we will explore how MHST coefficients can be used to identify important relationships within graphs, with the goal of narrowing the training space for attention mechanisms and graph transformers in large-scale problems. 

\section*{Acknowledgments}
This research was partially supported by the U.S.~National Science Foundation grants DMS-1912747 and CCF-1934568, and by the U.S.~Office of Naval Research grant N00014-20-1-2381. N.~S.\ also thanks St\'ephane Mallat and Gabriel Peyr\'e for their hospitality and support during his sabbatical stay (Oct.\ 2021--Jan.\ 2022) at the Centre Sciences des Donn\'ees, \'Ecole Normale Sup\'erieure (Ulm); and
Joan Bruna and Gabriel Peyr\'e for inviting him to the workshop ``A Multiscale tour of Harmonic Analysis and Machine Learning To Celebrate Stéphane Mallat's 60th birthday,'' (Apr.~2023) at the Institut des Hautes \'Etudes Scientifiques (IHES), Bures-sur-Yvette, where a preliminary version of this work was presented.

\appendix

\section{A Hierarchical Bipartition Tree of a Graph}
\AppendixLetterLabel{app:HBT}
In this appendix, we summarize our construction of a \emph{hierarchical
bipartition tree} of a given graph $G=(V,E)$ with $|V|=n$. We will describe it
for the $C_0$ case, but it is quite straightforward to handle the
$C_\kk$ ($\kk \geq 1$) case.
We impose the following four requirements for a hierarchical bipartition tree:
\begin{enumerate}
  \renewcommand{\labelenumi}{\roman{enumi}.}
\item \label{RPentire} The coarsest level is the entire graph; that is, $G_0^0 = G$.
\item \label{RPsingle} At the finest level, each region is a single node; that is, $n^\pmax_k = 1$ for $0 \leq k < K^\pmax = n$.  
\item \label{RPdisjoint} All regions on a given level are disjoint; that is, $V^p_k \cap V^p_{\tilde k} = \emptyset$ if $k \ne \tilde k$.  
\item \label{RPmustsplit} Each region on level $p < \pmax$ containing two or more nodes is partitioned into exactly two regions on level $p+1$.  
\end{enumerate}
\renewcommand{\theenumi}{\arabic{enumi}}
One method for generating a suitable recursive partitioning of a graph is to repeatedly partition the graph and subgraphs according to the signs of their respective Fiedler vectors of $\Lrw$ matrices. However, any other method to bipartition a given graph can be used.

\section{Brief Description of the HGLET and the GHWT}
\AppendixLetterLabel{app:HGLET-GHWT}
Using a recursive partitioning of the graph, the \emph{Hierarchical Graph
Laplacian Eigen Transform} (HGLET) generates an overcomplete dictionary in
which the supports of the basis vectors range from a single node to the entire
graph~\cite{IRION-SAITO-HGLETS,IRION-SAITO-SPIE,IRION-SAITO-TSIPN}.
We use $\bphi^p_{k,l}$ to denote the HGLET basis vectors, and we
use $c^p_{k,l}$ to denote the corresponding expansion coefficients.  As in the
recursive partitioning, let $p \in [0,\pmax]$ and $k \in [0,K^p)$ denote,
respectively, the level and region to which a basis vector/coefficient
corresponds. Let $l \in [0,n^p_k)$ index the vectors/coefficients associated
with $G^p_k$ The basis vectors are obtained by computing Laplacian
eigenvectors on each subgraph $G^p_k$ and extending them by zeros to the
entire graph. These may be the extended eigenvectors of $L$, $\Lrw$, or
$\Lsym$ on $G^p_k$.  One advantage of considering all three dictionaries is
the ability to construct a hybrid basis. For example, the basis vectors for
$G^p_k$ may be taken as the eigenvectors of $\Lsym(G^p_k)$, whereas those for
$G^{p'}_{k'}$ may be drawn from $\Lrw(G^{p'}_{k'})$.
In \cite{IRION-SAITO-SPIE}, we have demonstrated the use of hybrid bases for
simultaneous segmentation, denoising, and compression of classical
one-dimensional signals.  The computational complexity of the HGLET is
$O(n^3)$, arising from the computation of the full set of eigenvectors of
the $n \times n$ Laplacian matrix at level $p = 0$. When such a cost is
prohibitively high, the HGLET may instead be constructed only on subgraphs
$\{G^p_k\}$ with $n^p_k \leq \nmax < n$ nodes where, $\nmax$ is a
user-specified parameter reflecting the computational budget. In this case,
the complexity is reduced to $O(\nmax^2 n)$.  

Like the HGLET, the \emph{Generalized Haar--Walsh Transform} (GHWT) also employs a
recursive partitioning of the graph to generate an overcomplete dictionary,
but in this case the basis vectors are piecewise constant on their supports~~\cite{IRION-SAITO-GHWT,IRION-SAITO-SPIE,IRION-SAITO-TSIPN}. We denote denote the
GHWT basis vectors and expansion coefficients by $\bpsi^p_{k,l}$ and $d^p_{k,l}$,
respectively.
As with the HGLET, $p \in [0,\pmax]$ and $k \in [0,K^p)$ denote level and region,
respectively.  In the GHWT framework, $l$ is referred to as the basis vector's
(or coefficient's) \emph{tag}, taking $n^p_k$ distinct values in the range
$[0,2^{\pmax-p})$.  Coefficients with tag $l=0$ are called
  \emph{scaling coefficients}, those with $l=1$ are \emph{Haar coefficients},
  and those with $l \geq 2$ are \emph{Walsh coefficients}.
  For a hierarchical tree with $O(\log n)$ levels, the computational complexity
  of the GHWT is $O(n \log n)$.  

  A key feature of the GHWT is that its coefficients can be arranged in two ways.
  At each level $p$, grouping by their index $k$ yields the
  \emph{coarse-to-fine} (C2F) dictionary, which has the same structure as the
  HGLET dictionary.  Alternatively, grouping by the tag $l$ produces the
  \emph{fine-to-coarse} (F2C) dictionary, which admits ONBs not obtainable from
  the C2F dictionary---for example, the true Haar ONB can be only
  formed only in the F2C dictionary via a graph-adapted version of the
  best-basis search algorithm~\cite{COIF-WICK}.
  See also \cite[Sect.~6.2]{SAITO-SCHONSHECK-SHVARTS} and the accompanying
  figures for visual comparisons between the C2F and F2C GHWT dictionaries.
  For a graph with $n$ nodes the HGLET, GHWT-C2F, and GHWT-F2C
  dictionaries each contain more than $2^{\lfloor n/2 \rfloor}$ choosable bases
  (See Table~6.1 in \cite{IRION-PHD});
  exceptions may occur when the recursive partitioning is highly imbalanced.

  We now briefly describe theoretical properties of these dictionaries
  that are advantageous for signal approximation. These will be stated not in
  the original form given in \cite{sharon2015class, IRION-SAITO-TSIPN}
  for node signals, but in the updated form for signals on the $\kk$-simplices
  $C_\kk$ for $\kk \in \Znn$, as described in \cite{SAITO-SCHONSHECK-SHVARTS}.
 
  First, for singleton $\kk$-elements $\sigma$ and $\tau$ of $C_\kk$, and signal
  $\f$, we define a distance function and then the associated H\"{o}lder
  semi-norm as follows:
\begin{equation*}
    d(\sigma, \tau) \define \min \left\{n^p_k \, \Big\vert \,  \sigma, \tau \in C^p_k \right\}, \quad C_H(\f) \define \sup_{\sigma \neq \tau} \frac{ \left\vert [\f]_\sigma-[\f]_\tau  \right\vert }{d(\sigma, \tau)^\alpha}
\end{equation*}
where $\alpha \in (0,1]$ is a fixed constant.
Here, $d(\sigma, \tau)$ is the number of elements in the finest partition
in the tree that constrains both $\sigma$ and $\tau$. 
With these definitions, the dictionary coefficient decay and approximation
results of \cite{sharon2015class, IRION-SAITO-TSIPN} for the GHWT and HGLET
extend directly to the $\kk$-GHWT and $\kk$-HGLET bases, as detailed 
in \cite{SAITO-SCHONSHECK-SHVARTS}. 

\begin{theorem}[Fast Decay of Dictionary Coefficients as Scale Decreases]
For a simplicial complex $C$ equipped with a hierarchical bipartition tree, suppose that a signal $\f$ is H\"older continuous with exponent $\alpha \in (0,1]$ and constant $C_H(\f)$. Then the coefficients with $l \geq 1$ for the $\kk$-HGLET ($c^p_{k,l}$) and $\kk$-GHWT ($d^p_{k,l}$) satisfy:
\begin{equation*}
    \left\vert c^p_{k,l} \right\vert \leq C_H(\f) \left(n^p_k\right)^{\alpha+\frac{1}{2}}, \quad  \left\vert d^p_{k,l}  \right\vert \leq C_H(\f) \left(n^p_k\right)^{\alpha+\frac{1}{2}} .
\end{equation*}
\end{theorem}
\begin{proof}
See Theorem~3.1 of \cite{IRION-SAITO-TSIPN}.
\end{proof}
\begin{remark}
This theorem implies that
\begin{equation*}
  \left|c^p_{k,l}\right|, \left|d^p_{k,l}\right| \leq C_H(\f) \( n^p_k \)^{\alpha+\half} 
  \lessapprox C_H(\f) \( n/2^p \)^{\alpha+\half} \propto 2^{-p(\alpha+\half)}
  \propto 2^{j(\alpha+\half)},
\end{equation*}
for a typical hierarchical bipartition tree of a given simplicial complex
$C_\kk$.
Thus, these dictionary coefficients decay exponentially as the
\emph{partition parameter} $p$ increases and the \emph{scale parameter}
$j$ decreases: the finer the scale and the larger the $\alpha \in (0,1]$,
the faster the decay rate.
Equivalently, ``the smoother the underlying function, the more rapidly these
coefficients decay at finer scales.'' This behavior parallels that of
the conventional wavelet transforms for the H\"older continuous functions;
see, e.g., \cite[Sect.~6.1.3]{MALLAT-BOOK3} for details.
\end{remark}

Once an ONB from the $\kk$-HGLET or $\kk$-GHWT has been selected---by any method,
such as the graph-adapted version of the best-basis search algorithm~\cite{COIF-WICK})---the following result is useful for signal approximation in that basis.
\begin{theorem} [Approximation Power of Multiscale Bases]
Let $\{\bphi_l \}^{n-1}_{l=0}$ be a fixed ONB, and let $0 < \rho < 2$. Then, 
\begin{equation*}
    \|\f - P_m \f \|_2 \leq \frac{\vert \f \vert_\rho}{m^\beta}.
\end{equation*}
where $P_m\f$ denotes the best nonlinear $m$-term approximation of $\f$ in the
basis $\{\bphi_l \}^{n-1}_{l=0}$, $\beta=\frac{1}{\rho} - \frac{1}{2}$, and
$\vert \f \vert_\rho \define \left(\sum_{l=0}^{n-1} \vert \langle \f, \bphi_l \rangle \vert^\rho \right)^{\frac{1}{\rho}}$.
\end{theorem}
\begin{proof}
See Theorem~3.2 of \cite{IRION-SAITO-TSIPN} and Theorem~6.3 of \cite{sharon2015class}.
\end{proof}

\section{Description of Domain Classification Datasets}
\label{sec:datasets}
\paragraph{Google COLLAB} \cite{yanardag2015deep}
A collaboration dataset of 5K scientific papers represented as graphs. The goal is to predict which of three subfields a paper belongs to given its authors. The average number of nodes of these graphs is 45 while the number of nodes in the smallest graph is 32 and that of the largest graph is 492.

\paragraph{DD}\cite{dobson2003distinguishing}
A dataset of 1178 molecular graphs of medium to large size. The goal of the classification problem is to predict whether the protein is enzyme or not. The average number of nodes is 285 while the number of nodes of the smallest graph is 30 and that of the largest graph is 5748.

\paragraph{IMDB-B and IMDB-M} \cite{yanardag2015deep}
A pair of two datasets representing collaboration within popular movies from the Internet Movie Database with 1K (IMDB-B) and 1.5k (IMDB-M) graphs. Each graph represents members of cast and the goal is to predict the genre (Action/Romance for IMDB-B and Comedy/Romance/Sci-Fi of IMDB-M). The average, the smallest, and the largest number of nodes in IMDB-B are 20, 12, and 136 while those in IMDB-M are 13, 7, and 89.

\paragraph{MUTAG} \cite{debnath1991structure}
A dataset of 188 molecular graphs representing  mutagenic aromatic and heteroaromatic nitro compound. The dataset contains node-based features (detailing atom type), but we use only the adjacency matrix for our classification tests as in \cite{gao2019geometric}. The binary classification goal is to predict whether or not a molecule (graph) has a mutagenic effect on bacterium. The average number of nodes of these graphs is 18 while the number of nodes of the smallest graph is 10 and that of the largest graph is 28.

\paragraph{Proteins} \cite{borgwardt2005protein}
A dataset of 1113 molecular graphs. The goal of the classification problem is to predict whether the protein is enzyme or not. The task here is the same as in the DD dataset, but the molecules are both smaller and more similar to each other (see \cite{borgwardt2005protein} for details). The average number of nodes is 39 while the number of nodes of the smallest graph is 4 and that of the largest graph is 620.

\paragraph{PTC} \cite{toivonen2003statistical}
A dataset of 344 molecular graphs. The goal is to predict whether or not the compound is known to be carcinogenic in rats. The average number of nodes of these graphs is 26 while the number of nodes of the smallest graph is 6 and that of the largest graph is 109.

\bibliographystyle{elsarticle-num}
\bibliography{main.bib}

\end{document}